\documentclass[journal]{IEEEtranTIE}
\usepackage{graphicx}
\usepackage{cite}
\usepackage{picinpar}
\usepackage{amsmath}
\usepackage{url}
\usepackage{flushend}
\usepackage[latin1]{inputenc}
\usepackage{colortbl}
\usepackage{subfigure}
\usepackage{soul}
\usepackage{multirow}
\usepackage{pifont}
\usepackage{color}
\usepackage{alltt}
\usepackage[hidelinks]{hyperref}
\usepackage{enumerate}
\usepackage{breakurl}
\usepackage{epstopdf}
\usepackage{pbox}
\usepackage{cancel}
\usepackage{caption2}
\newtheorem{theorem}{\bf Theorem}
\newenvironment{proof}{{\it Proof}.}{\hfill $\square$\par}
\newenvironment{assumption}{{\it Assumption}.}{}
\usepackage{amssymb}
\usepackage{amsmath,bm}
\usepackage{makecell}
\usepackage{booktabs}
\usepackage[boxed,ruled,commentsnumbered]{algorithm2e}  
\definecolor{blue-green}{rgb}{0.0, 0.87, 0.87}
\definecolor{dollarbill}{rgb}{0.52, 0.73, 0.4}
\definecolor{applegreen}{rgb}{0.55, 0.71, 0.0}
\definecolor{carnelian}{rgb}{0.7, 0.11, 0.11}
\definecolor{mygray}{gray}{.9}
\begin{document}

\title{A Learning Framework for $n$-bit Quantized Neural Networks toward FPGAs}
\author{Jun~Chen, Liang~Liu, Yong~Liu,~\IEEEmembership{Member,~IEEE}, Xianfang~Zeng
	
	\thanks{Jun~Chen, Liang Liu, Yong~Liu and Xianfang~Zeng are with the Institute of Cyber-Systems and Control, Zhejiang University, Hangzhou,
		China, 310027, e-mail: yongliu@iipc.zju.edu.cn.}
	\thanks{Yong~Liu is the corresponding author.}
}

\maketitle
	
\begin{abstract}
The quantized neural network (QNN) is an efficient approach for network compression and can be widely used in the implementation of FPGAs. This paper proposes a novel learning framework for $n$-bit QNNs, whose weights are constrained to the power of two. To solve the gradient vanishing problem, we propose a reconstructed gradient function for QNNs in back-propagation algorithm that can directly get the real gradient rather than estimating an approximate gradient of the expected loss. We also propose a novel QNN structure named $n$-BQ-NN, which uses shift operation to replace the multiply operation and is more suitable for the inference on FPGAs. Furthermore, we also design a shift vector processing element (SVPE) array to replace all 16-bit multiplications with SHIFT operations in convolution operation on FPGAs. We also carry out comparable experiments to evaluate our framework. The experimental results show that the quantized models of ResNet, DenseNet and AlexNet through our learning framework can achieve almost the same accuracies with the original full-precision models. Moreover, when using our learning framework to train our $n$-BQ-NN from scratch, it can achieve state-of-the-art results compared with typical low-precision QNNs. Experiments on Xilinx ZCU102 platform show that our $n$-BQ-NN with our SVPE can execute 2.9 times faster than with the vector processing element (VPE) in inference. As the SHIFT operation in our SVPE array will not consume Digital Signal Processings (DSPs) resources on FPGAs, the experiments have shown that the use of SVPE array also reduces average energy consumption to 68.7\% of the VPE array with 16-bit.
\end{abstract}


\markboth{IEEE TRANSACTIONS ON NEURAL NETWORKS AND LEARNING SYSTEMS}%
{}

\definecolor{limegreen}{rgb}{0.2, 0.8, 0.2}
\definecolor{forestgreen}{rgb}{0.13, 0.55, 0.13}
\definecolor{greenhtml}{rgb}{0.0, 0.5, 0.0}

\section{Introduction}

\IEEEPARstart{D}{eep} convolutional neural networks (CNNs) have substantially become the dominant Artificial Intelligence (AI) approach for a variety of computer vision tasks such as image classification \cite{Krizhevsky2012ImageNet,Szegedy2014Going,Simonyan2015VeryDC}, face recognition \cite{taigman2014deepface,sun2014deep}, semantic segmentation \cite{long2015fully,chen2014semantic} and object detection \cite{girshick2015fast,ren2015faster}. The significant accuracy improvement of CNNs brings with the cost of huge computational complexity, resource, and power consumption as it requires a comprehensive estimation of all the scopes within the feature maps \cite{russakovsky2015imagenet,deng2009imagenet}. For example, the AlexNet model is over 200 MB, and the VGG-16 model is over 500 MB \cite{russakovsky2015imagenet}. Towards such overwhelming resources and computation pressure, hardware accelerators such as GPUs, FPGAs, and ASICs have been applied to accelerate CNNs. Among these accelerators, FPGAs have emerged as one of the popular solutions when considering both the reprogramability and energy efficiency.

Implementing CNN on FPGAs is not an efficient practice due to limited resources and bandwidth. Thus QNN is a good choice for FPGAs implementation, which simultaneously gives consideration to computational efficiency, resources and classification accuracy in inference. In general, QNNs can be achieved in two ways: $1)$ An estimator is used to estimate the gradient of the expected loss to solve the problem of gradient vanishing so that QNNs can be trained from scratch with the help of this estimator. $2)$ Fine-tuning on a pretrained full-precision model obtains QNNs that bypasses the problem of gradient vanishing. Although the first method estimates a gradient, which makes it possible to train QNNs from scratch, the gradient of expected loss obtained by estimators has a noise source compared to the real gradient that causes a gap on classification accuracy between the QNNs and full-precision CNNs. The second method fine-tunes QNNs on a pretrained full-precision model that solves the problem of classification accuracy better, but a challenging factor is that the structure of QNNs is limited by the original structure of the pretrained CNNs model, and the structure of QNNs cannot be flexibly adjusted. Due to the constraints of computational resources and computational efficiency on FPGAs, it is inevitable to adjust the network structure for the hardware environment. In order to transform different CNNs into QNNs that can run efficiently on FPGAs, it is essential for a general learning framework to solve the above two challenges 
\cite{peemen2013memory,meloni2016high,sankaradas2009massively,Courbariaux2016BinarizedNN}.

In this paper, we propose a novel learning framework for $n$-bit QNNs, whose weights are constrained to the power of two $(\pm2^{-0},\pm2^{-1},\cdots,0)$. We introduce a reconstructed gradient function in back-propagation algorithm that can directly get the real gradient, rather than the estimated gradient given by estimators. Thus the QNNs trained by our framework will be more accurate. At the same time, QNNs after adjusting the structure can continue to fine-tune with our framework. The learning framework is applied to train our proposed $n$-BQ-NN, which is suitable for efficient implementation on FPGAs. We also evaluate the effectiveness of our approach on state-of-the-art networks such as ResNet~\cite{he2016deep}, DenseNet~\cite{huang2017densely} and AlexNet~\cite{Krizhevsky2012ImageNet}. The main contributions of this article are summarized as follows:
\begin{enumerate}
	\item We propose a novel learning framework for $n$-bit QNNs. In this framework, we propose a reconstructed gradient function in back-propagation algorithm, which can overcome the gradient vanishing problem during training the QNNs and can calculate the accurate gradient compared with the estimators based approaches. We achieve state-of-the-art results compared with typically low-precision QNNs.
	\item We propose a highly efficient QNN structure called $n$-BQ-NN for FPGAs. Our proposed architecture, which consists entirely of convolutional layers and implements a uniform convolution kernel, can maximize the resource utilization and improve the parallel computational efficiency on FPGAs while preserving the accuracy of QNNs.
	\item We propose a novel shift vector processing element (SVPE) array for FPGAs, which replaces the multiplication with the SHIFT operation when calculating convolution operation on FPGAs. The computational efficiency of our SVPE array can achieve a performance of 2.9 times higher than that of the VPE array in the case of the same network structure on FPGAs.
\end{enumerate}

The rest of this paper is organized as follows: Section~II summarizes related prior works on QNNs and FPGAs. Our learning framework is presented in Section~III. In Sections~IV, we demonstrate the effectiveness of our learning framework via comparable experiments. We theoretically analyse and practically test the computational efficiency of our $n$-BQ-NN using our quantization method in Section~V. The conclusion is given in Section~VI.

\section{Related work}

\subsection{Learning for QNNs}

Since the amount of the model capacity is too large, it is necessary to cut down it to perform CNNs on FPGAs, which is consistent with the purpose of deep compression. In general, deep compression can be divided into three categories, i.e., pruning, Huffman coding, and quantization. 
The pruning method will simplify the deep neural network by cutting off the network connections with small weights on the normal trained network~\cite{han2015learning}~\cite{li2016pruning}~\cite{han2015deep}. The Huffman coding method is an optimal code used for lossless data compression \cite{van1976construction} which uses entropy to encode source symbols by variable-length codewords. Han \emph{et al}. \cite{han2015deep} show that 20\% - 30\% of the network storage will be saved after Huffman coding the non-uniformly distributed values. When considering perform compressed networks on FPGAs, the network after pruning is an asymmetric structure, which is unsuitable for hardware implementation, and the Huffman coding may only be regarded as a post-compression combined with the other two compression methods, so most of the hardware accelerators will focus on the quantization method.

The quantization based method normally employ the low-precision weights, varied from 1 bit to 5 bits, to represent the CNNs~\cite{Courbariaux2016BinarizedNN,choi2018bridging,choi2018pact,liu2018bi,zhuang2018towards,zhou2017incremental,zhou2016dorefa,zhu2016trained}. Some studies train QNNs from scratch by estimating the gradient of expected loss based on Straight-Through Estimator \cite{Courbariaux2016BinarizedNN,choi2018pact,choi2018bridging,zhou2016dorefa}. For example, Courbariaux \emph{et al}. \cite{Courbariaux2016BinarizedNN} train a classification neural network from scratch with 1-bit weight and activation, which can run 7 times faster than the CNNs. Choi \emph{et al}. \cite{choi2018pact} propose a neural quantization scheme called Parameter Clipping Activation, which uses a parameter to find the optimal quantization scale for arbitrary bit width activations. Choi \emph{et al}. \cite{choi2018bridging} introduce a novel technique called Statistics-Aware Weight Binning, which finds the optimal scaling factor based on statistical characteristics of the distribution of the weights to minimize the quantization error. The QNNs trained by the above quantization methods only accelerate the inference, Zhou \emph{et al}. \cite{zhou2016dorefa} propose a DoReFa-Net that can accelerate both training and inference by low bit width weights, activations and gradients respectively. However, these estimator-based methods have a noise compared to the real gradient. Thus these QNNs can't achieve an ideal classification accuracy, especially on multi-classification datasets such as CIFAR-100.

Some other quantization methods are dedicated to design special strategies to fine-tune QNNs, which will not rely on the backpropagation algorithm and can bypass the problem of gradient vanishing~\cite{park2018precision,zhou2017incremental,zhu2016trained}. They can achieve a much better accuracy as they are independent of estimators. For example,  Park \emph{et al}. \cite{park2018precision} propose precision highway that has an end-to-end high-precision information flow for ultra-low-precision computation. This linear weight quantization method is based on the assumption that the weight distribution is the Laplace distribution. Recently, Zhou \emph{et al}. \cite{zhou2017incremental} propose an incremental network quantization method, which converts pretrained full-precision CNNs model into a low-precision model where the weights are constrained to the power of two or zero. It has been studies that there will be little loss on the classification accuracies when using 2-5 bits low-precision weight \cite{zhou2016dorefa,zhou2017incremental}. However, these quantization methods will depend on the pretrained network structure rather than the backpropagation algorithm, which will be difficult to satisfy the network-structure-optimization requirements due to the hardware limitation.

\subsection{CNNs Implemented by FPGAs}

Considering the inference, the CNNs have a highly hierarchical structure of multiple feature maps, whose structure exposes a large amount of parallelism that makes CNNs very suitable for FPGAs implementation. This structure builds on the accumulation of a huge number of convolutions that will consume a huge number of floating-point resources on FPGAs. In addition, the structure of CNNs often contains many convolutional layers. Thus the convolution module with different parameters needs to be executed iteratively during the inference. Frequent execution of data caching and parameter loading will be limited by the bandwidth. Therefore, in many studies, their hardware structures of CNNs are designed mainly for the two bottlenecks of floating-point resources and bandwidth~\cite{peemen2013memory,huang2017densely,he2016deep,meloni2016high}.

In terms of optimizing for floating-point resources, Lu \emph{et al}. \cite{lu2017evaluating} design a fast Winograd algorithm, which can decrease the use of floating-point resources on FPGAs and reduces the complexity of convolution dramatically. Simultaneously, they also give the formula for estimating the computational efficiency, which demonstrates that the fast Winograd algorithm is more efficient than conventional convolutional algorithm due to the use of fewer floating-point resources on FPGAs. Meloni \emph{et al}. \cite{meloni2016high} present an accelerator configuration for CNNs that reaches more than 97\% DSP resource utilization at 150 MHz operating frequency with 16-bit precision. And they show that the floating-point resource utilization is the highest when executing 3$\times$3 filters on FPGAs.

Other studies have focused on optimizing the data scheduling structure to reduce the impact of the bandwidth. For example, Sankaradas \emph{et al}. \cite{sankaradas2009massively} implement a vector processing element (VPE) array coprocessor, which can accelerate the CNNs by optimizing the cache between distributed off-chip memory banks and on-chip computing elements on FPGAs. Peemen \emph{et al}. \cite{peemen2013memory} show that their scheduler prefers to use only convolutional layers without fully connected layers on FPGAs, which can maximize the efficiency of on-chip memories by reducing the impact of the bandwidth bottleneck.

The crucial issue with the above methods is that they usually only consider the bottleneck at a single level and fail to coordinate these two constraints to improve the computational efficiency of the hardware accelerators. In this paper, we reduce the impact of the above two constraints by introducing the QNNs into FPGAs, which provides a new idea to deal with the above two bottlenecks. Since the weights in our QNNs are quantized to the power of two, the quantized weights directly reduce the bandwidth required to load the weights. In addition, the use of the quantized weights can translate the multiplication into shifting in convolution module, which greatly reduces the use of floating-point resources.

\section{$n$-BQ-NN}
\subsection{Fundamental Idea of Our $n$-BQ-NN}

\begin{figure}[h]
	\centering
	\includegraphics[width=.5\textwidth]{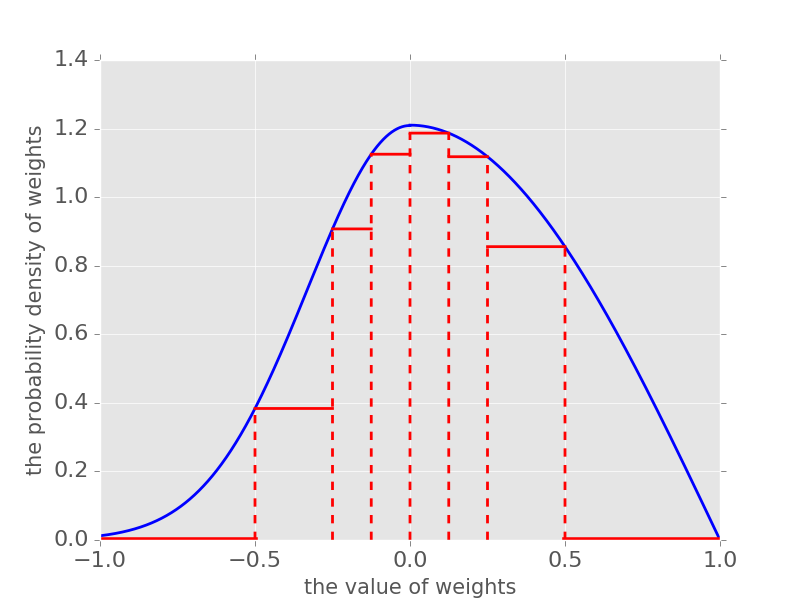}
	\caption{For any full-precision weight distribution (indicated by the blue curve in the figure) trained by CNNs model, non-uniform sampling can be used to approximate the full-precision distribution which represented by the red curve in the figure.}
	\label{sample} 
\end{figure}

The main idea of our $n$-BQ-NN is based on \textbf{Fig}.\ref{sample}, which shows the information loss led by the quantization method with the power of $2$ can be interpreted as the sampling loss caused by non-uniform sampling. In fact, the weights of CNNs with large absolute values will be dominant to the overall classification accuracy of the networks, although these weights with large values only account for a small ratio among all the weights~\cite{deng2009imagenet,ioffe2015batch}. For an arbitrary probability density function of the weights in a neural network, denoted as $\phi(x)$, we can use the blue curve in \textbf{Fig}.\ref{sample} to represent $\phi(x)$ which meets $\int_{-1}^{1} \phi(x) dx = 1$. In this way, we can calculate the sampling loss $\Phi(x)$ which can be represented as the area between two distributions (red and blue curves) in \textbf{Fig}.\ref{sample}. By calculating, the sampling loss $\Phi(x)$ is represented as the following recursive formula, where $n$ is the quantized bit width of the weights.
\begin{equation} 
\begin{split}
\left\{ 
\begin{array}{cll} 
\Phi(1) = 1 - \phi(2^{-1}) & n=1 \\ 
\Phi(n) = \Phi(n-1) + 2^{1-n}\phi(2^{1-n}) \\ - 2^{1-n}\phi(2^{-n}) & n>1
\end{array} \right. 
\end{split}
\label{note}
\end{equation}

It can be seen from the above formula that the sampling loss always decreases as the increasing of quantized bit width of the weights, which indicates that the sampling loss is negatively related to the quantized bit width of weights. However, the bit width is limited and needs to reduce as much as possible in QNNs. Thus, finding the best balance between quantized bit width and the sampling loss is the key to balancing the performance, speed, and resources of QNNs. We define the $\mathcal{L}(n) = 2^{1-n}[\phi(2^{1-n}) - \phi(2^{-n})]$ as the variation between two sampling losses, $\Phi(n)$ and $\Phi(n-1)$, from \textbf{Eq}.\ref{note}. Then, we can prove that $\mathcal{L}(4)$ will approach to zero in our quantization method with the power of $2$, which can be ensured by the \textbf{Theorem}~\ref{thm1}. Therefore, continuing to increase the quantized bit width of $n$ after $3$ is not helpful to decrease the sampling loss.

\begin{theorem}
\label{thm1}
$0 < \left| \mathcal{L}(4) \right| < 7.8\times10^{-3}$
\end{theorem}

\begin{proof}
We use Taylor expansion with Peano residuals to represent the probability density function $\phi(x)$,
\begin{equation} 
\begin{split}
& \phi(2^{-n}) = \phi(0) - ln2 \phi'(0) n 2^{-n} + o(n 2^{-n}) \\
& \phi(2^{1-n}) = \phi(0) - ln2 \phi'(0) n 2^{1-n} + o(n 2^{1-n}) \\
\end{split}
\label{taylor}
\end{equation}
Substituting \textbf{Eq}.\ref{taylor} into $\mathcal{L}(n)$, we get
\begin{equation} 
\begin{split}
& \mathcal{L}(n) = 2^{1-n}[\phi(2^{1-n}) - \phi(2^{-n})] \\
& = 2^{1-n}[ln2 \phi'(0) n 2^{-n} - \cancel{o(n 2^{-n})} - ln2 \phi'(0) n 2^{1-n} \\
& + \cancel{o(n 2^{1-n})}] \\
& \approx - ln2 \phi'(0) n 2^{1-2n}
\end{split}
\label{error}
\end{equation}
Since $0 < \Phi(n) < 1$, we deduce that $0 < \left| \mathcal{L}(2) \right| < 1$ and $0 < ln2 \phi'(0) < \frac{1}{4}$. In final, we get $0 < \left| \mathcal{L}(4) \right| < \frac{1}{128}$ by substituting the range of $ln2 \phi'(0)$ into $\mathcal{L}(4)$
due to the \textbf{Eq}.\ref{error}.
\end{proof}

From the hardware perspective, the resource consumption of SHIFT operation is much less than multiplication, so our intention is to use the SHIFT operation instead of multiplication. Considering that the shift right operation will make the weights exceed the constraint range of ($-1$,$1$), thus, all SHIFT operations are shift left and every quantized weight is chosen from the entries $(\pm2^{-0},\pm2^{-1},\cdots,\pm2^{-i},0)$, where $\pm2^{-i}$ indicates its multiplication can be calculated by $<<i$ and $0$ indicates that no operations are required. Our $n$-BQ-NN quantizes the weights to the entries, which are encoded to $n$-bit and suitable for hardware computation. Under such circumstance, the staircase function $\operatorname{staircase}(W)$ can be used to describe our $n$-bit quantized weights as \textbf{Eq}.\ref{eq1} (typically, $n$ is greater than $1$, and $\operatorname{staircase}(W)$ is degraded to $\operatorname{sign}(W)$ if $n$ is equal to $1$), where $W$ are full-precision weights.
\begin{equation} 
\operatorname{staircase}(W)=\left\{ 
\begin{array}{ccc} 
2^{-i} \operatorname{sign}(W) & & \Delta_{i+1} \leq |W| < \Delta_i \\ 
0 & & |W| < \Delta_r
\end{array} \right. 
\label{eq1}
\end{equation}
Here $i$ is taken from $r-1$ to $0$ in turn, where $r = 2^{n-1} - 1$ and $\operatorname{sign}(W)$ is the sign function:
\begin{equation} 
\operatorname{sign}(W)=\left\{ 
\begin{array}{clc} 
+1 & & W \geq 0 \\ 
-1 & & W < 0
\end{array} \right. 
\end{equation}

\subsection{Gradients Computation in $n$-BQ-NN}

In order to facilitate the discussion as follows, we need to define some variables first, where $W^l_{jk}$ represents the weight that connects the $k$-th neuron of the $(l-1)$-th layer to the $j$-th neuron of the $l$-th layer, $b^l_j$ represents the bias of the $j$-th neuron of the $l$-th layer, $z^l_j$ represents the input of the $j$-th neuron of the $l$-th layer ($z^l_j = \sum_k W^l_{jk} a^{l-1}_k + b^l_j$), $a^l_j$ represents the output of the $j$-th neuron of the $l$-th layer ($a^l_j = \theta(z^l_j)$), and $\theta$ is activation function.

We also have to add a extra quantized weight so that we can train our $n$-BQ-NN, where the quantized weight is shown as follows,
\begin{equation} 
\hat{W}^l_{jk} = \operatorname{staircase}(W^l_{jk})
\label{eq9}
\end{equation}
And the cost function of mini-batch of $m$ samples in our $n$-BQ-NN is,
\begin{equation} 
C=\frac{1}{2 m} \sum_{x}\left\|y(x)-a^{L}(x)\right\|^{2}
\label{eq4}
\end{equation}
Where $x$ is the input sample, $y$ is the actual classification, $a^L$ is the prediction output, and $L$ is the maximum number of layers in the network.

By defining $\mathcal{T}_{j}^{l} \equiv \frac{\partial C}{\partial z_{j}^{l}}$ as the error produced by the $j$-th neuron of the $l$-th layer, we can use the back-propagation algorithm to calculate the gradient and update the parameters according to the following three steps\footnote{$\odot$ represents the Hadamard product that is used for point-to-point product between matrices or vectors.}.

\begin{enumerate}
	\item[-]{Calculating the error of the last layer of the network.
	\begin{equation}
	\begin{aligned} &\mathcal{T}_{j}^{L}=\frac{\partial C}{\partial z_{j}^{L}}=\frac{\partial C}{\partial a_{j}^{L}} \cdot \frac{\partial a_{j}^{L}}{\partial z_{j}^{L}} \\ 
	&{\bm{\mathcal{T}}^{L}=\frac{\partial C}{\partial \textbf{a}^{L}} \odot \frac{\partial \textbf{a}^{L}}{\partial \textbf{z}^{L}}=\nabla_{\textbf{a}} C \odot \theta^{\prime}\left(\textbf{z}^{L}\right)}\end{aligned} 
	\label{backpropagation1}
	\end{equation}
}
	\item[-]{Calculating the error of each layer of the network from the back to the front.
	\begin{equation}
	\begin{aligned} \mathcal{T}_{j}^{l}=\frac{\partial C}{\partial z_{j}^{l}} &=\sum_{k} \frac{\partial C}{\partial z_{k}^{l+1}} \cdot \frac{\partial z_{k}^{l+1}}{\partial a_{j}^{l}} \cdot \frac{\partial a_{j}^{l}}{\partial z_{j}^{l}} \\ &=\sum_{k} \mathcal{T}_{k}^{l+1} \cdot \frac{\partial\left(\hat{W}_{k j}^{l+1} a_{j}^{l}+b_{k}^{l+1}\right)}{\partial a_{j}^{l}} \cdot \theta^{\prime}\left(z_{j}^{l}\right) \\ &=\sum_{k} \mathcal{T}_{k}^{l+1} \cdot \hat{W}_{k j}^{l+1} \cdot \theta^{\prime}\left(z_{j}^{l}\right) \\ \bm{\mathcal{T}}^{l}=&\left(\left(\hat{\textbf{W}}^{l+1}\right)^{T} \bm{\mathcal{T}}^{l+1}\right) \odot \theta^{\prime}\left(\textbf{z}^{l}\right) \end{aligned}
	\label{backpropagation3}
	\end{equation}
}
	\item[-]{Calculating the gradient of weight and bias respectively.
	\begin{equation}
	\begin{aligned}
	g^b&=\frac{\partial C}{\partial b_{j}^{l}}=\frac{\partial C}{\partial z_{j}^{l}} \cdot \frac{\partial z_{j}^{l}}{\partial b_{j}^{l}} \\ &=
	\mathcal{T}_{j}^{l} \cdot \frac{\partial\left(\hat{W}^{l}_{jk} a_{k}^{l-1}+b_{j}^{l}\right)}{\partial b_{j}^{l}}=\mathcal{T}_{j}^{l}
	\end{aligned}
	\end{equation}
	\begin{equation}
	\begin{aligned}
	g^W&=\frac{\partial C}{\partial W_{j k}^{l}}=\frac{\partial C}{\partial z_{j}^{l}} \cdot \frac{\partial z_{j}^{l}}{\partial W_{j k}^{l}}=\mathcal{T}_{j}^{l} \cdot \frac{\partial\left(\hat{W}_{j k}^{l} a_{k}^{l-1}+b_{j}^{l}\right)}{\partial W_{j k}^{l}} \\ &=\mathcal{T}_{j}^{l} \cdot \frac{\partial\left(\hat{W}_{j k}^{l} a_{k}^{l-1}+b_{j}^{l}\right)}{\partial \hat{W}_{j k}^{l}} \cdot \frac{\partial \hat{W}_{j k}^{l}}{\partial W_{j k}^{l}} \\
	&=\mathcal{T}_{j}^{l} \cdot a^{l-1}_k \cdot \frac{\partial \hat{W}_{j k}^{l}}{\partial W_{j k}^{l}}
	\label{backpropagation2}
	\end{aligned}
	\end{equation}
}
\end{enumerate}
\begin{figure}[htbp]
	\centering
	\includegraphics[width=.5\textwidth]{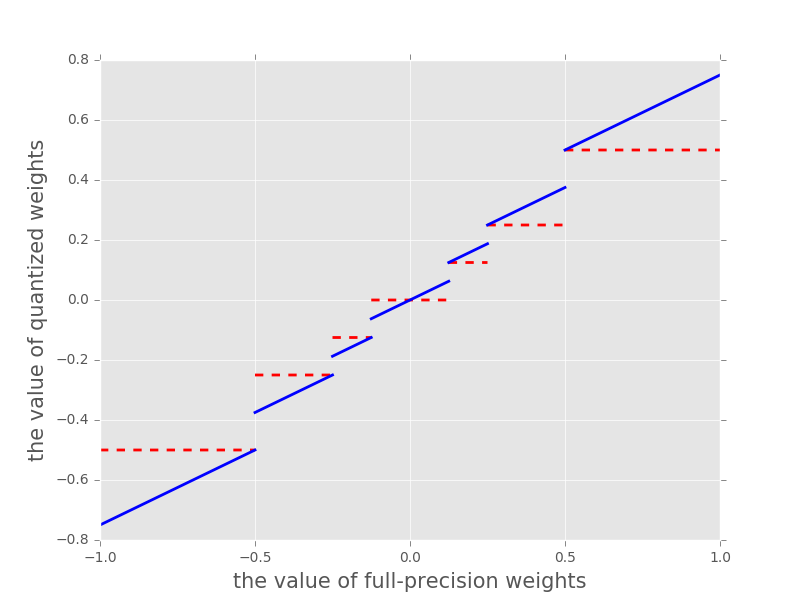}
	\caption{The red dotted line represents the staircase function, and the blue full line represents our reconstructed function.}
	\label{function} 
\end{figure}

In the above process of deriving the entire back-propagation, except for the gradient of weight of the last step, the other steps are well-defined. Based on the \textbf{Eq}.\ref{backpropagation2}, the gradient of weight can be calculated as follows,
\begin{equation} 
g^W = \mathcal{T}_{j}^{l} \cdot a^{l-1}_k \cdot \frac{\partial \hat{W}_{j k}^{l}}{\partial W_{j k}^{l}} = 0
\label{eq5}
\end{equation}
Where $\frac{\partial \hat{W}_{j k}^{l}}{\partial W_{j k}^{l}}$ is exactly the $\operatorname{staircase}'(W_{j k}^{l})$, which is the derivative of $\operatorname{staircase}(W_{j k}^{l})$. And this derivative satisfies the conditions of Dirac Delta Function $\delta(x)$. According to the properties of $\delta(x)$, $\frac{\partial \hat{W}_{j k}^{l}}{\partial W_{j k}^{l}}$ can be calculated as follows,
\begin{equation} 
\frac{\partial \hat{W}_{j k}^{l}}{\partial W_{j k}^{l}} = \delta(W_{j k}^{l}) = 0
\end{equation}
Substituting $\frac{\partial \hat{W}_{j k}^{l}}{\partial W_{j k}^{l}} = 0$ into \textbf{Eq}.\ref{eq5}, we discover that model cannot be trained by back-propagation algorithm due to gradient vanishing.

To resolve the above problem, we reconstruct the quantized weight function as \textbf{Eq}.\ref{eq6} to ensure that the weights can be updated by using the back-propagation algorithm as shown in \textbf{Fig}.\ref{function}, the blue full line, where $\alpha$ is an adjustable parameter in the range of $(0,1)$.
\begin{equation} 
\tilde{W}^l_{jk} = (1-\alpha) \hat{W}^l_{jk} + \alpha W^l_{jk}
\label{eq6}
\end{equation}

By substituting \textbf{Eq}.\ref{eq6} into \textbf{Eq}.\ref{backpropagation2}, we can recalculate the gradient of weight as follows again with $\frac{\partial \tilde{W}_{j k}^{l}}{\partial W_{j k}^{l}} = \alpha + (1-\alpha) \delta(W^l_{jk}) = \alpha$,
\begin{equation} 
g^W = \mathcal{T}_{j}^{l} \cdot a^{l-1}_k \cdot \frac{\partial \tilde{W}_{j k}^{l}}{\partial W_{j k}^{l}} = \alpha \mathcal{T}_{j}^{l} \cdot a^{l-1}_k
\label{eq8}
\end{equation}

At this point, we have reconstructed the quantized weight function as \textbf{Eq}.\ref{eq6} to solve the gradient vanishing, but the weights cannot be quantized to the entries $(\pm2^{-0},\pm2^{-1},\cdots,0)$ directly as \textbf{Eq}.\ref{eq9}. However, we can prove that the reconstructed quantized weight function will approximate to the entries after several iterations, which can be ensured by the \textbf{Theorem}~\ref{thm2}.

\begin{assumption}
Since the algorithm needs to be iterated, our problem needs to be discussed within the framework of the series. We define $W^l_{jk}$ as an iteration of $x_n$, $\tilde{W}^l_{jk}$ is equivalent to $x_{n+1}$, and the value of $a_j$ is chosen from $\hat{W}^l_{jk}=\operatorname{staircase}(W^l_{jk}) = (\pm2^{-0},\pm2^{-1},\cdots,\pm2^{-i},0)$.
\end{assumption}

\begin{theorem}
\label{thm2}
In the framework of the series, $\tilde{W}^l_{jk}$ will approach to $\hat{W}^l_{jk}$ when the number of iterations is sufficient, when $n$ is the number of iterations.
\end{theorem}

\begin{proof}
The general terms of series $x$ from $1$ to $n$ are written as follows based on \textbf{Eq}.\ref{eq6},

\begin{equation} 
\left\{\begin{array}{ccc}
& x_2 - \alpha x_1 = (1-\alpha) a_j& (1) \\
& \vdots & \vdots \\
& x_n - \alpha x_{n-1} = (1-\alpha) a_j & (n-1) \\
& x_{n+1} - \alpha x_n = (1-\alpha) a_j & (n) \\
\end{array} \right.
\label{series}
\end{equation}
We let $\alpha^{(n-1)} \times (1) + \cdots + \alpha \times (n-1) + (n)$, then we get the equation as follows, 

\begin{equation}
\begin{aligned} 
x_{n+1} - \alpha^{(n-1)} x_1 & = (1-\alpha) a_j(1 + \alpha + \cdots + \alpha^{(n-1)}) \\
& = a_j(1-\alpha^{(n-1)})
\end{aligned}
\end{equation} 

\begin{figure*}[t]
	\centering
	\includegraphics[width=1.0\textwidth]{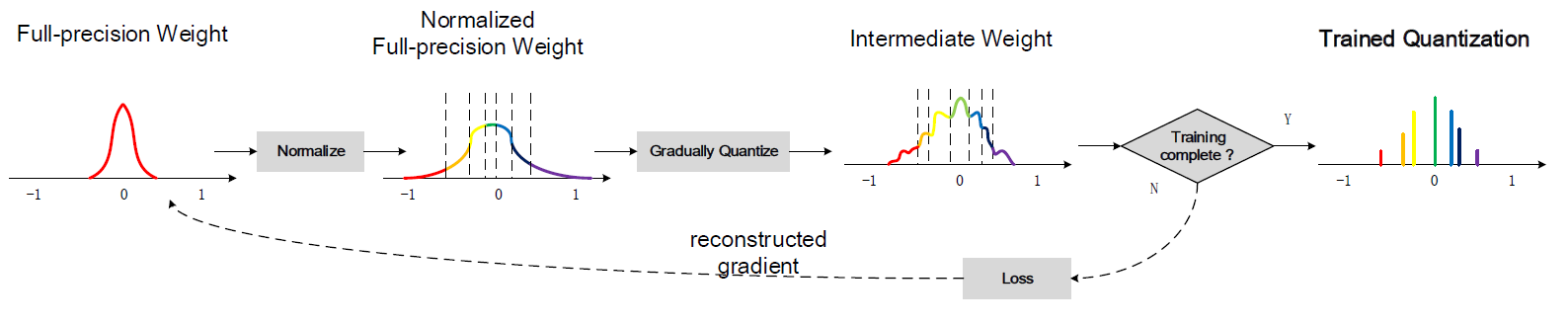}
	\caption{Overview of our trained quantization procedure.}
	\label{procedure} 
\end{figure*}

As the number of iterations increases, $x_{n+1}$ will approach $a_j$. With the guarantee of \textbf{Theorem}~\ref{thm2}, the above equation can be rewritten as $\tilde{W}^l_{jk} = \operatorname{staircase}(W^l_{jk})$ (namely, $x_{n+1}=a_j$) when the number of iterations is enough ($n \to \infty$) and $\alpha$ is in the range of $(0,1)$. In the actual algorithm implementation, it is only necessary to iterate through several steps following the training process, and the networks can be quantized completely as \textbf{Eq}.\ref{eq9}.
\end{proof}

The design of $\alpha$ in our reconstructed quantized weight function takes three aspects into consideration: First, the designed function must satisfy the \textbf{Theorem}~\ref{thm2}. Second, our reconstructed function indicates that the ratio of $1-\alpha : \alpha$ between quantized weights and full-precision weights can be used to adjust the information ratio of between quantized weights and full-precision weights in the training process. Third, on the other hand, $\alpha$ is the slope of our reconstructed function shown as the blue full line in \textbf{Fig}.\ref{function}, which can be used to change the gradient descent rate of back-propagation based on \textbf{Eq}.\ref{eq8} during the training.

\subsection{Posterior-Distribution Adjustment}

In the initialization of the networks, the initialization modes MSRA and Xavier \cite{he2015delving} which will adjust variance based on the number of inputs are prone to converge than the traditional Gaussian distribution initialization mode with fixed variance in DNNs. Inspiring by this fact, we suspect that adjusting the distribution of quantized weights may make it easier for us to train our $n$-BQ-NN. Here, we consider that full-precision networks are prone to converge than quantized networks; thus, we prefer to keep the distribution of quantized weights consistent with full-precision weights'. Comparing the probability density function before quantization $\phi(x)$ (its corresponding expectation and variance are $E(x)$ and $Var(x)$, respectively) and the probability density function after quantization $\operatorname{staircase}(x)$ (as \textbf{Eq}.\ref{eq1}), we make their expectation and variance equal respectively so that their distribution is consistent as follows,
\begin{equation} 
\left\{ 
\begin{array}{lcl} 
E(x) = \int_{-1}^{1} x \operatorname{staircase}(x) dx \\
Var(x) = \int_{-1}^{1} (x - E(x))^2 \operatorname{staircase}(x) dx
\end{array} \right. 
\label{eq10}
\end{equation}

The original full-precision probability density function $\phi(x)$ and the value of quantized weight function $\operatorname{staircase}(x)$ are fixed, so we can only adjust the value range of $\operatorname{staircase}(x)$ to meet \textbf{Eq}.\ref{eq10}.

\subsection{The Training Algorithm for $n$-BQ-NN}

In actual training algorithm for $n$-BQ-NN, the Batch Normalization (BN \cite{ioffe2015batch}) is added in our $n$-BQ-NN because it is conducive to reduce the overall impact of the weight scale and accelerate the training. Thus, we will derive the back-propagation algorithm for $n$-BQ-NN with BN and give the training algorithm in this subsection.

First, we define four variables of BN, where $\sigma$ represents the variance of all samples of a batch, $\mu$ represents the sample mean, and $\gamma,\beta$ are the scale variation coefficients. Due to the existence of BN, the bias term can be ignored, so the input of the neuron is re-expressed as $z^l_j = \sum_k W^l_{jk} a^{l-1}_k$, the normalized input of the neuron is $\hat{z}_j=\gamma \frac{z_j-\mu}{\sigma}+\beta$, and the output of the neuron is $a_j = \theta(\hat{z}_j)$. Then, we can calculate the error and the gradient, based on the discussion of Section III-B, according to the following three steps.

	\begin{enumerate}
		\item[-]{Counting the mean and variance of the sample, and calculating the gradient of them.
			\begin{equation}
			\begin{aligned} &\mu=\frac{1}{m}\sum_{j=1}^{m} z_j \\ 
			&\sigma^2=\frac{1}{m}\sum_{j=1}^{m} (z_j-\mu)^2\end{aligned} 
			\end{equation}
			\begin{equation}
			\begin{aligned} \frac{\partial C}{\partial \sigma^2} &=\sum_{k} \frac{\partial C}{\partial a_{k}} \frac{\partial a_{k}}{\partial \hat{z}_{k}} \frac{\partial \hat{z}_{k}}{\partial \sigma^2} \\ &=-\frac{1}{2} \gamma \sigma^{-3} \sum_{k} \frac{\partial C}{\partial a_k} \theta^{\prime}\left(z_{k}\right) \left(z_{k}-\mu\right) \end{aligned}
			\end{equation}
			\begin{equation}
			\begin{aligned} \frac{\partial C}{\partial \mu} &=\sum_{k} \frac{\partial C}{\partial a_{k}} \frac{\partial a_{k}}{\partial \hat{z}_{k}} \frac{\partial \hat{z}_{k}}{\partial \mu}+\frac{\partial C}{\partial \sigma^2} \frac{\partial \sigma^2}{\partial \mu} \\ &=-\frac{\gamma}{\sigma} \sum_{k} \frac{\partial C}{\partial a_{k}} \theta^{\prime}\left(z_{k}\right)-\frac{2}{m}\frac{\partial C}{\partial \sigma^2} \sum_{k}\left(z_{k}-\mu\right) \end{aligned}
			\end{equation}
		}
		\item[-]{Calculating the error of the network.
			\begin{equation}
			\begin{aligned} &\mathcal{T}_{j}=\frac{\partial C}{\partial z_{j}} = \frac{\partial C}{\partial a_{j}} \frac{\partial a_{j}}{\partial \hat{z}_{j}} \frac{\partial \hat{z}_{j}}{\partial z_{j}}+\frac{\partial C}{\partial \sigma^2}\frac{\partial \sigma^2}{\partial z_{j}}+\frac{\partial C}{\partial \mu}\frac{\partial \mu}{\partial z_{j}} \\ &=\frac{\gamma}{\sigma}\frac{\partial C}{\partial a_{j}} \theta^{\prime}\left(\hat{z}_{j}\right)+\frac{2}{m}\sum_{k}\left(z_{k}-\mu\right)\frac{\partial C}{\partial \sigma^2}+\frac{1}{m}\frac{\partial C}{\partial \mu}  \end{aligned}
			\end{equation}
		}
		\item[-]{Calculating the gradient of weight, $\gamma$, and $\beta$ respectively.
			\begin{equation}
			g^W = \mathcal{T}_{j} \cdot a_k \cdot \frac{\partial \tilde{W}_{j k}}{\partial W_{j k}^{l}} = \alpha \mathcal{T}_{j} \cdot a_k
			\end{equation}
			\begin{equation}
			g^{\gamma}=\sum_{k}\frac{\partial C}{\partial a_{k}} \frac{\partial a_k}{\partial \hat{z}_{k}} \frac{\partial \hat{z}_{k}}{\partial \gamma}=\sum_{k}\frac{\partial C}{\partial a_{k}}\theta^{\prime}\left(\hat{z}_{k}\right)\frac{z_k-\mu}{\sigma}
			\end{equation}
			\begin{equation}
			g^{\beta}=\sum_{k}\frac{\partial C}{\partial a_{k}} \frac{\partial a_k}{\partial \hat{z}_{k}} \frac{\partial \hat{z}_{k}}{\partial \beta}=\sum_{k}\frac{\partial C}{\partial a_{k}}\theta^{\prime}\left(\hat{z}_{k}\right)
			\end{equation}
		}
\end{enumerate}

With the foundation of the above formulas, we can propose our training algorithm for $n$-BQ-NN, as indicated in \textbf{Algorithm}~\ref{al}. This algorithm covers two learning modes: training from scratch and fine-tuning on the pretrained model, where the first mode means the weights are randomly initialized and the second mode means the weights are initialized by the pretrained full-precision network model. The overall quantization process is illustrated as \textbf{Fig}.\ref{procedure}. The code for training algorithm is available \footnote{\url{https://github.com/papcjy/n-BQ-NN}}.

\IncMargin{1em} 
\begin{algorithm}
\SetKwInOut{KwIn}{\textbf{Require}}
\SetKwInOut{KwOut}{\textbf{Ensure}}
\KwIn{a minibatch of outputs and targets $(a^L,y)$, learning rate $\eta$, previous weights $W^k$, previous BN parameters ($\gamma^k$,$\beta^k$), and a constant $\alpha$.} 
\KwOut{the updated weights ${(W^k)}^*$ and updated BN parameters $({(\gamma^k)}^*,{(\beta^k)}^*)$} 

\{1. Computing the parameter gradients:\}

\{1.1 Forward propagation:\}

\For{$k=1$ to $L$} 
{ 
	$\tilde{W}^k \leftarrow (1-\alpha) \operatorname{staircase}(W^k) + \alpha W^k$
	
	$z^k \leftarrow a^{k-1}\tilde{W}^k$	
	
	$\hat{z}^k \leftarrow \operatorname{BatchNorm}(z^k,\gamma^k,\beta^k)$
	
	$a^k \leftarrow \theta(\hat{z}^k)$	
} 
\{1.2 Backward propagation:\}

Computing $g^{a^L}=\frac{\partial C}{\partial a^L}$ based on $a^L$ and $y$.

\For{$k=L$ to $1$} 
{ 
	$(g^{\gamma^k},g^{\beta^k}) \leftarrow  \operatorname{BackBatchNorm}(g^{a^k},z^k,\gamma^k,\beta^k)$
	
	$\mathcal{T}^k \leftarrow g^{a^k} \theta^{\prime}(\hat{z}^{k})$
	
	$g^{a^{k-1}} \leftarrow \mathcal{T}^k \tilde{W}^k$
	
	$g^{W^k} \leftarrow \alpha {(\mathcal{T}^k)}^T a^{k-1}$
} 
\{2. Updating the parameter gradients:\}

\For{$k=1$ to $L$} 
{ 
	$({(\gamma^k)}^*,{(\beta^k)}^*) \leftarrow \operatorname{Update}(\gamma^k,\beta^k,\eta,g^{\gamma^k},g^{\beta^k})$
	
	${(W^k)}^* \leftarrow \operatorname{Update}(W^k, g^{W^{k}}, \eta)$
} 
	\caption{Training algorithm for $n$-BQ-NN with BN. $C$ is the cost function for mini-batch, $\theta$ is the activation function, and $L$ is the number of layers. The function $\operatorname{staircase}(\cdot)$ specifies how to quantize the weights. $\operatorname{BatchNorm}()$ specifies how to batch-normalize the inputs. $\operatorname{BackBatchNorm}()$ specifies how to back-propagate through the BN. $\operatorname{Update}()$ specifies how to update the parameters when their gradients are known, using either SGD or ADAM.}
	\label{al}
\end{algorithm}
\DecMargin{1em}

\subsection{Activation Quantization in $n$-BQ-NN}

The above discussion is all about the quantization of weights. To take the integrity of our $n$-BQ-NN and the necessity of subsequent ablation experiments into consideration, we need to discuss the quantization of activations in this subsection. Now, let's put our eyes back on Section III-B. In the case of the quantized activations, the output of the $j$-th neuron of the $l$-th layer can be rewritten as,

\begin{equation}
\hat{a}^l_j = \operatorname{staircase}(z^l_j)
\label{activation}
\end{equation}

Where $\operatorname{staircase}()$ is activation function.

At this point, we have encountered the same problem, the error of network $\mathcal{T}_{j}^{l}$ becomes zero due to the existence of $\frac{\partial \hat{a}_{j}^{l}}{\partial z_{j}^{l}}$, when the \textbf{Eq}.\ref{activation} is substituted into \textbf{Eq}.\ref{backpropagation1} and \textbf{Eq}.\ref{backpropagation3}.

Considering the expectation of $\frac{\partial C}{\partial z_{j}^{L}}$, the error of network has reappeared, which is guaranteed by \textbf{Theorem}~\ref{thm3}.

\begin{theorem}
	\label{thm3}
	Let us define $C=C(\hat{a}_j,\epsilon_j)$ where $\hat{a}_j$ follows \textbf{Eq}.\ref{activation} that is chosen from $(\pm2^{-0},\pm2^{-1},\cdots,0)$, then, we get a new expression as
	\begin{equation}
	\mathbb{E}_{\epsilon_j}\left[\frac{\partial C}{\partial z_j}\right]=\lambda\frac{\partial C}{\partial \hat{a}_j}, \operatorname{if} |z_j| \leq 1,
	\label{estimator}
	\end{equation} Where $\epsilon_j$ is the noise source that influences $z_j$, $\mathbb{E}_{\epsilon_j}[\cdot]$ means the expectation over $z_j$, and $\lambda$ is a constant.
\end{theorem}

\begin{proof}
	\begin{equation}
	\begin{split}
	&\mathbb{E}_{\epsilon_j}\left[\frac{\partial}{\partial z_j} C\right]=\frac{\partial}{\partial z_j} \mathbb{E}_{\epsilon_j}[C] \\ 
	&=\frac{\partial}{\partial z_j}[\sum_{i}C(\hat{a}_j=+2^i) P(z_j>\epsilon_j | z_j)+\\
	&\sum_{i}C(\hat{a}_j=-2^i)(1-P(z_j>\epsilon_j |z_j))] \\ 
	&=\frac{\partial P(z_j>\epsilon_j | z_j)}{\partial z_j}[\sum_{i}C(\hat{a}_j=+2^i)-\sum_{i}C(\hat{a}_j=-2^i)]
	\label{key}
	\end{split}
	\end{equation}	
	For $C(\hat{a}_j=\pm2^i)$, we can approximate it using the Taylor expansion.
	\begin{equation}
	\begin{split}
	&C(\hat{a}_j=+2^i))=C(\hat{a}_j=0)+\frac{\partial C}{\partial \hat{a}_j}\bigg|_{\hat{a}_j=0} 2^i+\\
	&\frac{\partial^{2} C}{\partial \hat{a}_j^{2}}\bigg|_{\hat{a}_j=0} 2^{2i}+O\left(\frac{\partial^{3} C}{\partial \hat{a}_j^{3}}\bigg|_{\hat{a}_j=0} 2^{3i}\right)\\
	&C(\hat{a}_j=-2^i))=C(\hat{a}_j=0)-\frac{\partial C}{\partial \hat{a}_j}\bigg|_{\hat{a}_j=0} 2^i+\\
	&\frac{\partial^{2} C}{\partial \hat{a}_j^{2}}\bigg|_{\hat{a}_j=0} 2^{2i}+O\left(\frac{\partial^{3} C}{\partial \hat{a}_j^{3}}\bigg|_{\hat{a}_j=0} 2^{3i}\right)
	\label{L}
	\end{split}
	\end{equation}	
	For $\frac{\partial P(z_j>\epsilon_j | z_j)}{\partial z_j}$, we split it into two parts.
	\begin{equation}
	\begin{split}
	&\frac{\partial P(z_j>\epsilon_j | z_j)}{\partial z_j}=\underset{|z_j|>1}{\frac{\partial P(z_j>\epsilon_j | z_j)}{\partial z_j}}+\underset{|z_j|\leq 1}{\frac{\partial P(z_j>\epsilon_j | z_j)}{\partial z_j}}\\
	&=\frac{\partial \int_{-1}^{1}\frac{1}{2}\,d\epsilon_j}{\partial z_j}+\frac{\partial \int_{-z_j}^{z_j}\frac{1}{2}\,d\epsilon_j}{\partial z_j}=\mathbf{1}_{|z_j| \leq 1}
	\label{partical}
	\end{split}
	\end{equation}
	Combining \textbf{Eq}.\ref{L} and \textbf{Eq}.\ref{partical}, the \textbf{Eq}.\ref{key} can be derived as
	\begin{equation}
	\mathbb{E}_{\epsilon_j}\left[\frac{\partial C}{\partial z_j}\right]=\mathbf{1}_{|z_j| \leq 1}\left(2\sum_{i}2^{2i}\frac{\partial C}{\partial \hat{a}_j}\bigg|_{\hat{a}_j=0}\right)
	\label{ste2}
	\end{equation}
	Let $2\sum_{i}2^{2i}=\lambda$, then
	\begin{equation}
	\mathbb{E}_{\epsilon_j}\left[\frac{\partial C}{\partial z_j}\right]=\lambda\frac{\partial C}{\partial \hat{a}_j}\mathbf{1}_{|z_j| \leq 1}
	\label{STE}
	\end{equation}
\end{proof}	

Under the \textbf{Theorem}~\ref{thm3}, we can re-express the error of network and quantize the activations in our $n$-BQ-NN by rewriting the \text{Eq}.\ref{backpropagation1} and \text{Eq}.\ref{backpropagation3} as follows,

\begin{equation}
\mathcal{T}_{j}^{L}=\frac{\partial C}{\partial z_{j}^{L}}=\lambda\frac{\partial C}{\partial \hat{a}_{j}^{L}}\mathbf{1}_{|z_j| \leq 1}
\end{equation}
\begin{equation}
\begin{aligned}
\mathcal{T}_{j}^{l} &=\frac{\partial C}{\partial z_{j}^{l}}=\lambda\frac{\partial C}{\partial \hat{a}_{j}^{l}}\mathbf{1}_{|z_j| \leq 1}=\lambda\sum_{k} \frac{\partial C}{\partial z_{k}^{l+1}} \cdot \frac{\partial z_{k}^{l+1}}{\partial \hat{a}_{j}^{l}} \\ &=\lambda\sum_{k} \mathcal{T}_{k}^{l+1} \cdot \frac{\partial\left(\hat{W}_{k j}^{l+1} \hat{a}_{j}^{l}+b_{k}^{l+1}\right)}{\partial \hat{a}_{j}^{l}} \\ &=\lambda\sum_{k} \mathcal{T}_{k}^{l+1} \cdot \hat{W}_{k j}^{l+1}\mathbf{1}_{|z_j| \leq 1}
\end{aligned}
\end{equation}

\section{Experiment}

\renewcommand\arraystretch{1.5}
\begin{table}
	\caption{The outline of the proposed $n$-BQ-NN network architecture. Here, taking the CIFAR datasets as an example, the initial input size of the network is $32\times32\times3$. The conv quantized contains three calculation steps, respectively, $\hat{W}=\operatorname{staircase}(W)$, $\operatorname{net}=\operatorname{conv2d}(\hat{W},x)$, and $\operatorname{net}=\operatorname{BatchNorm}(\operatorname{net})$, where the weights involved in convolution calculation are quantized weights that are chosen from the entries $(\pm2^{-0},\pm2^{-1},\cdots,0)$. In convolution calculation, the multiplications are replaced by SHIFT operations during the inference, because the weights are power of 2.}
	\label{TBQ-Net_structure}
	\begin{center}
		\begin{tabular}{|l|c|c|}
			\hline
			\textbf{type} & \textbf{patch size/stride} & \textbf{output size} \\
			\hline
			\hline
			conv quantized & 3$\times$3/1 & 32$\times$32$\times$128 \\
			\hline
			conv quantized & 3$\times$3/1 & 32$\times$32$\times$128 \\
			\hline
			conv quantized & 3$\times$3/1 & 32$\times$32$\times$128 \\
			\hline
			pool & 2$\times$2/2 & 16$\times$16$\times$128 \\
			\hline
			conv quantized & 3$\times$3/1 & 16$\times$16$\times$256 \\
			\hline
			conv quantized & 3$\times$3/1 & 16$\times$16$\times$256 \\
			\hline
			conv quantized & 3$\times$3/1 & 16$\times$16$\times$256 \\
			\hline
			pool & 2$\times$2/2 & 8$\times$8$\times$256 \\
			\hline
			conv quantized & 3$\times$3/1 & 8$\times$8$\times$512 \\
			\hline
			conv quantized & 1$\times$1/1 & 8$\times$8$\times$1024 \\
			\hline
			conv quantized & 1$\times$1/1 & 8$\times$8$\times$10 (100) \\
			\hline
			pool & 8$\times$8 & 1$\times$1$\times$10 (100) \\
			\hline
			softmax & classifier & 1$\times$1$\times$10 (100) \\
			\hline
		\end{tabular}
	\end{center}
\end{table}

In our experiments, we use three network structures ResNet, DenseNet and AlexNet. The network structure of our $n$-BQ-NN ($n$ can take $1,2,3,4,5$) is similar to the architecture of All-CNN \cite{springenberg2014striving} that consists solely of convolution layers and Network in Network block \cite{lin2013network}. \textbf{Table}~\ref{TBQ-Net_structure} details the parameter settings and our network architecture. In the following experiments, our training algorithm is used to train the model from scratch or fine-tune on the full-precision model in five benchmark datasets MNIST, SVHN, CIFAR-10, CIFAR-100 and ImageNet. We unfold our experiments from 4 dimensions, respectively classification accuracy compared with low-precision QNNs, quantization errors by our training method, compression ratio in different datasets, and convergence speed compared with BNN.

\subsection{MNIST}

The MNIST dataset \cite{lecun1998gradient} consists of handwritten digit images with 32$\times$32 pixels, organized into 10 classes (0 to 9). The training and test sets contain 60,000 and 10,000 images respectively. We perform this dataset without data augmentation~\cite{lee2015deeply}.

\subsection{CIFAR}

The two CIFAR datasets \cite{krizhevsky2009learning} consist of natural color images with 32$\times$32 pixels, respectively 50,000 training and 10,000 test images, and we hold out 5,000 training images as a validation set from the training set. CIFAR-10 (C10) consists of images organized into 10 classes and CIFAR-100 (C100) into 100 classes. We adopt a standard data augmentation scheme (random corner cropping and random flipping) that is widely used for these two datasets \cite{srivastava2015training,huang2016deep,larsson2016fractalnet,lin2013network,romero2014fitnets,lee2015deeply,springenberg2014striving,srivastava2015training}. We normalize the images using the channel means and standard deviations in preprocessing.

\subsection{SVHN}
The SVHN dataset \cite{netzer2011reading} consists of color images of house numbers collected by Google Street View with 32$\times$32 pixels, organized into 10 classes (0 to 9). There are 73,257 images in the training set, 531,131 images for additional training, and 26,032 images in the test set respectively. We divide the pixel values by 255.0 so that they are in the [0,1] range as \cite{zagoruyko2016wide}. Moreover, we do not preprocess the dataset following common practice without data augmentation  \cite{goodfellow2013maxout,huang2016deep,lin2013network,lee2015deeply,sermanet2012convolutional}.

\subsection{Experiment Results}

\subsubsection{$n$-bit}

\renewcommand\arraystretch{1.5}
\begin{table}[h]
	\caption{Our $n$-BQ-ResNet generates extremely low-precision models with very similar accuracy compared with full-precision ResNet-110 model on CIFAR-10.}
	\label{nbit} 
	\begin{center}
		\begin{tabular}{c c c}
			\toprule[1pt]
			\textbf{Model} & \textbf{Bit-width} & \textbf{Test error} \\
			\midrule
			ResNet-110 ref & 16 & 6.61\% \\
			$n$-BQ-ResNet & 5 & 7.04\% \\
			$n$-BQ-ResNet & 4 & 7.07\% \\
			$n$-BQ-ResNet & 3 & 7.15\% \\
			$n$-BQ-ResNet & 2 & 8.76\% \\
			$n$-BQ-ResNet & 1 & 10.52\% \\
			\bottomrule[1pt]
		\end{tabular}
	\end{center}
\end{table}
As the theoretical analysis in Section III, different quantized bit width brings different sampling loss, and the larger bit width means the less sampling loss. Thus, in this experiment, we evaluate the test error rates of our $n$-BQ-ResNet that is fine-tuned on full-precision ResNet-110 when $n$ takes different values on CIFAR-10. The experimental results from \textbf{Table}~\ref{nbit} are consistent with \textbf{Eq}.\ref{note}. Therefore, the choice of 3-bit is better because $\mathcal{L}(4)$ is close to $0$ as \textbf{Eq}.\ref{error} when considering both the sampling loss and the conciseness of weight representation. Obviously, this result is also experimentally proved by works in \cite{zhou2017incremental}. Thus, our $n$-BQ-NN is chosen as T-BQ-NN when $n=3$ in the subsequent experiments.

\subsubsection{Accuracy and Capacity}

\renewcommand\arraystretch{1.5}
\begin{table*}[h]
	\caption{Error rates on CIFAR-10 and CIFAR-100 datasets with standard data augmentation (translation and mirroring). Error rates on MNIST and SVHN datasets without data augmentation. The overall best results are \textbf{bold}. ``*" represents the results run by our implementation, the rest of results represents that these results are directly cited from the paper in the front of the row.}
	\label{test_error}
	\begin{center}
		\begin{tabular}{|c c c | c c c c|}
			\hline
			 & & & \multicolumn{4}{c|}{\textbf{Test error}} \\
			\hline
			\textbf{Method} & \textbf{Depth} & \textbf{Params} & CIFAR-10 & CIFAR-100 & SVHN & MNIST \\
			\hline
			Network in Network \cite{lin2013network} & 9 & 1.9M & 8.81\% & 35.68\% & 2.35\% & 0.53\% \\
			All-CNN \cite{springenberg2014striving} & 9 & 1.4M & \textbf{7.25\%} & 33.71\% & $^*$3.17\% & $^*$0.63\% \\
			Highway Network \cite{srivastava2015training} & 19 & 2.3M & 7.72\% & 32.39\% & $^*2.61$\% & 0.67\% \\
			\hline
			BNN \cite{Courbariaux2016BinarizedNN} & 9 & 1.7M & 11.40\% & $^*42.13$\% & 2.80\% & 0.96\% \\
			\hline
			Round Quantization & 9 & 1.2M & 85.88\% & 98.90\% & 83.72\% & 80.55\% \\ 
			T-BQ-NN & 9 & 1.2M & 7.59\% & \textbf{28.90\%} & \textbf{2.29\%} & \textbf{0.50\%} \\
			\hline
		\end{tabular}
	\end{center}
\end{table*}

\renewcommand\arraystretch{1.5}
\begin{table}[htpb]
	\caption{Fine-tunning ResNet and DenseNet by our training algorithm on CIFAR10(100) and SVHN. Where the results on C10 and C100 with data augmentation and the results on SVHN without data augmentation.}
	\label{quantization_error} 
	\begin{center}
		\begin{tabular}{|c|c c c c c|}
			\hline
			& \textbf{Network} & \textbf{Depth} & \textbf{Bit-width} & \textbf{Params} & \textbf{Test error(\%)} \\
			\hline
			\multirow{4}*{\rotatebox{90}{CIFAR-10}} & ResNet & 110 & 16 & 1.7M & 6.61 \\
			& T-BQ-ResNet & 110 & 3 & 0.3M & 7.15 (+0.54) \\
			& DenseNet & 100 & 16 & 0.8M & 4.51 \\
			& T-BQ-DenseNet & 100 & 3 & 0.15M & 5.31 (+0.80) \\
			\hline
			\multirow{4}*{\rotatebox{90}{CIFAR-100}} & ResNet & 110 & 16 & 1.7M & 35.87 \\
			& T-BQ-ResNet & 110 & 3 & 0.3M & 37.56 (+1.69) \\
			& DenseNet & 100 & 16 & 0.8M & 22.27 \\
			& T-BQ-DenseNet & 100 & 3 & 0.15M & 24.10 (+1.83) \\
			\hline
			\multirow{4}*{\rotatebox{90}{SVHN}} & ResNet & 110 & 16 & 1.7M & 3.13 \\
			& T-BQ-ResNet & 110 & 3 & 0.3M & 3.25 (+0.12) \\
			& DenseNet & 100 & 16 & 0.8M & 1.76 \\
			& T-BQ-DenseNet & 100 & 3 & 0.15M & 2.10 (+0.34) \\
			\hline
		\end{tabular}
	\end{center}
\end{table}

\renewcommand\arraystretch{1.5}
\begin{table}[h]
	\caption{Deep compression method for T-BQ-ResNet and T-BQ-DenseNet. P: Pruning, Q: Quantization, H: Huffman coding.}
	\label{compression_ratio} 
	\begin{center}
		\begin{tabular}{c c c}
			\toprule[1pt]
			\textbf{Method} & \textbf{\makecell*[c]{Encoding \\ bit-width}} & \textbf{Compression ratio} \\
			\midrule
			T-BQ-ResNet on C10 (P+Q) & 3 & 49$\times$ \\
			T-BQ-ResNet on C10 (P+Q+H) & 2.6 & 57$\times$ \\
			T-BQ-ResNet on C100 (P+Q) & 3 & 25$\times$ \\
			T-BQ-ResNet on C100 (P+Q+H) & 2.8 & 27$\times$ \\
			T-BQ-ResNet on SVHN (P+Q) & 3 & 24$\times$ \\
			T-BQ-ResNet on SVHN (P+Q+H) & 2.8 & 26$\times$ \\
			
			T-BQ-DenseNet on C10 (P+Q) & 3 & 38$\times$ \\
			T-BQ-DenseNet on C10 (P+Q+H) & 2.5 & 46$\times$ \\
			T-BQ-DenseNet on C100 (P+Q) & 3 & 15$\times$ \\
			T-BQ-DenseNet on C100 (P+Q+H) & 2.8 & 16$\times$ \\
			T-BQ-DenseNet on SVHN (P+Q) & 3 & 133$\times$ \\
			T-BQ-DenseNet on SVHN (P+Q+H) & 2.1 & 190$\times$ \\
			\bottomrule[1pt]
		\end{tabular}
	\end{center}
\end{table}

\renewcommand\arraystretch{1.8}
\begin{table*}[h]
	\caption{Comparison of classification accuracy on the test set for ImageNet with different bitwidths of weights and activations. Single-crop evaluation results top-1 and top-5 accuracy are given based on AlexNet. Note the gray results are implemented by our $n$-BQ-NN where the training method of 1-bit activations is introduced in Section III-E, and other results are reported by \cite{sze2017efficient}. We quantize the same layers of AlexNet to low-precision, as BNN~\cite{Courbariaux2016BinarizedNN}, BC~\cite{courbariaux2015binaryconnect}, TWN~\cite{li2016ternary} and DoReFa-Net~\cite{zhou2016dorefa} do.}
	\label{alexNet} 
	\begin{center}
		\begin{tabular}{p{40pt}p{50pt}p{70pt}p{70pt}p{70pt}}
			\toprule[1pt]
			\textbf{$n$-bit Weights} & \textbf{$n$-bit Activations} & \textbf{Inference Operation} & \textbf{AlexNet Top-1 Accuracy} & \textbf{AlexNet Top-5 Accuracy} \\
			\midrule			
			1 & 1 & XNOR & 0.279 (BNN) & 0.504 (BNN) \\
			\rowcolor{mygray}1 & 1 & XNOR & 0.348 & 0.601 \\
			\hline
			1 & 32 (float) & XNOR ADDER & 0.368 (BC) & 0.620 (BC) \\
			\hline
			\rowcolor{mygray}1 & 16 (float) & XNOR ADDER & 0.486 & 0.734 \\
			\hline
			\hline
			2 & 32 (float) & XNOR ADDER & 0.529 (TWN) & 0.766 (TWN) \\
			\rowcolor{mygray}2 & 16 (float) & XNOR ADDER & 0.536 & 0.777 \\
			\hline
			\hline
			\rowcolor{mygray}3 & 16 (float) & SHIFT ADDER & \textbf{0.560} & \textbf{0.795} \\
			\hline
			\hline
			8 (float) & 8 (float) & MAC & 0.530 (DoReFa-Net) & 0.768 (DoReFa-Net) \\
			\hline
			32 (float) & 32 (float) & MAC & 0.566 & 0.802 \\
			\bottomrule[1pt]
		\end{tabular}
	\end{center}
\end{table*}

As hardware devices require relatively simple architecture and less number of layers, we have selected some network suited for hardware implementation as our comparative experiment. For example, BNN with binary weights and activations replaces most multiplications by 1-bit XNOR operations. Network in Network utilizes the global average pooling over feature maps in the classification layer, which is less prone to overfitting than the fully connected layers. All-CNN achieves a new architecture that consists solely of convolution layers replacing max-pooling by a convolution layer without loss in accuracy on several benchmarks. Highway Network allows unimpeded information flow across many layers using adaptive gating units to regulate the information flow. 

There is a general manifestation that T-BQ-NN performs better than most other network structures, while these network structures have never been quantized except BNN. In the experiment here, T-BQ-NN is trained by our training algorithm from scratch due to the lack of pretrained model, and this model is trained with a mini-batch size of 50, a weight decay of 0.0001. Its test error rates of 7.59 \% on CIFAR-10,  28.9 \% on CIFAR-100, 2.29 \% on SVHN, and 0.5 \% on MNIST are lower than the test error rates achieved by Network in Network, Highway Network, and BNN. Particularly, T-BQ-NN makes up for the classification accuracy of BNN on CIFAR-100 to some extent. The best result for all listed datasets is T-BQ-NN except CIFAR-10 is All-CNN, and all results are shown in \textbf{Table}~\ref{test_error}.

Our model capacity is even more encouraging: the number of parameters of T-BQ-NN is significantly lower than other network structures. Particularly, T-BQ-NN achieves the number of parameters of 1.2M that is even lower than 1.7M of BNN shown in \textbf{Table}~\ref{test_error}.

\subsubsection{Extension}

\begin{figure*}[h]
	\centering
	\includegraphics[width=1.0\textwidth]{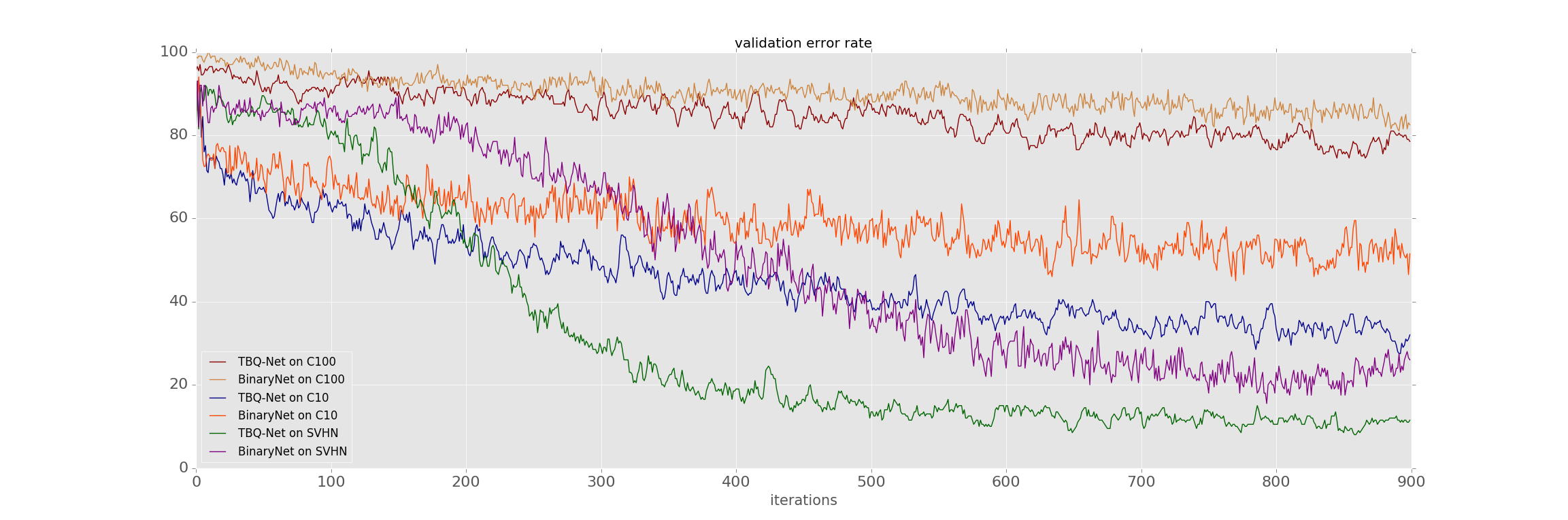}
	\caption{Comparison of the T-BQ-NN and BNN top-1 error rates on CIFAR10, CIFAR100, and SVHN validation datasets. Where the curves, from top to bottom at 900 iterations, represent T-BQ-NN on C100, BNN on C100, T-BQ-NN on C10, BNN on C10, T-BQ-NN on SVHN, and BNN on SVHN, respectively. Note that the results of datasets run by ourselves.}
	\label{convergence_speed} 
\end{figure*}

One positive-effect of our training algorithm is universal. We popularize our training method to the better and deeper architectures, not just limited to CNNs, such as ResNet \cite{he2016deep} and DenseNet \cite{huang2017densely}. In the experiment here, T-BQ-ResNet and T-BQ-DenseNet are 3-bit that are fine-tuned by our training algorithm based on the full-precision model of ResNet and DenseNet.

For T-BQ-ResNet, all the multiplications are converted to SHIFT and ADDER operations using 3-bit weights in all convolutional layers and shortcut connections. We use a momentum of 0.9, weight decay of 0.0001 \cite{gross2016training,goodfellow2013maxout}, and adopt the weight initialization and BN \cite{ioffe2015batch,he2015delving} without dropout \cite{srivastava2014dropout}. This model is trained with a mini-batch size of 128 and a learning rate of 0.1, divided by 10 at 32k and 38k iterations, and terminates training at 64k iterations. We achieve the test error rates of 7.15\% on C10, 37.56\% on C100, and 3.25\% on SVHN using T-BQ-ResNet, just rises 0.54\% on C10, 1.69\% on C100, and 0.12\% on SVHN compared with ResNet on the basis of~\textbf{Table}~\ref{quantization_error}.

For T-BQ-DenseNet, its model consists of Bottleneck layers indicated to BN-ReLU-Conv(1 $\times$ 1)-BN-ReLU-Conv(3 $\times$ 3) and transition layers indicated to BN-ReLU-Conv(1 $\times$ 1)-averpool(2 $\times$ 2), and both of these layers contain 1x1 convolution. We use a weight decay of 0.0001, momentum of 0.9 \cite{sutskever2013importance}, and adopt the weight initialization and BN without dropout. This model is trained with an initial learning rate of 0.1, divided by 10 at 50\% and 75\% of the total number of training epochs. And we train using a batch size of 64 for 300 and 40 epochs respectively on CIFAR and SVHN. Comparing between DenseNet and T-BQ-DenseNet, the increasing in error is 0.80\% from 4.51\% to 5.31\% on C10, 1.83\% from 22.27\% to 24.10\% on C100, and 0.34\% from 1.76\% to 2.10\% on SVHN as~\textbf{Table}~\ref{quantization_error}.

We attribute this primarily to reduce the number of parameters approximately 5 times from 0.8M to 0.15M on T-BQ-DenseNet, and from 1.7M to 0.3M on T-BQ-ResNet is shown as \textbf{Table}~\ref{quantization_error}. Furthermore, a hybrid network compression solution combined with three different techniques, respectively network pruning \cite{li2016pruning}, quantization and Huffman coding is tested for T-BQ-ResNet and T-BQ-DenseNet in a scene with the same classification accuracy. Comparing with the original full-precision ResNet-110 model, we achieve the compression ratio of 57$\times$ on C10, 27$\times$ on C100 and 26$\times$ on SVHN for T-BQ-ResNet. For T-BQ-DenseNet, the compression ratio is 46$\times$ on C10, 16$\times$ on C100 and 190$\times$ on SVHN shown as~\textbf{Table}~\ref{compression_ratio}.

\subsubsection{Convergence Speed}
In this experiment, we train our T-BQ-NN and BNN from scratch on C10, C100, and SVHN. The results in \textbf{Fig}.\ref{convergence_speed} indicate that T-BQ-NN not only has better performance on classification accuracy than BNN but also converges much faster. We just only compare our method with BNN, because the weights of other network models in \textbf{Table}~\ref{test_error} are full-precision and these models are not quantized except BNN and T-BQ-NN. We use the same conditions, including learning rate, batch size, and iterations, to test the error rates of BNN and T-BQ-NN at first epoch. Comparing with BNN, T-BQ-NN reaches the best test error dropping from 500 to 150 epochs on C10, from 1000 to 100 epochs on MNIST, and from 1000 to 180 epochs on C100. As a result, T-BQ-NN can be trained much easier and faster than BNN.

This result may be due to the fact that straight-through estimator used by BNN contains noise, which causes the unexpected deviation, while our training algorithm is based entirely on back-propagation without the effect of noise and weight representation is more abundant.

\subsection{ImageNet}

We further evaluate our $n$-BQ-NN on ILSVRC2012 \cite{deng2009imagenet} image classification dataset that consists of 1.2 million high-resolution natural images where the validation set contains 50k images. These images are organized into 1000 categories of the object for training, which are resized to 224$\times$224 pixels before fed into the network. In the next experiments, we report our single-crop evaluation results using top-1 and top-5 accuracy.

\textbf{AlexNet:} This CNN architecture is the first structure that shows to be successful on ImageNet classification task, which consists of 5 convolutional layers and 2 fully-connected layers~\cite{Krizhevsky2012ImageNet}. We use AlexNet coupled with BN that contains a total of 61M parameters.

In training, images are resized randomly to 256$\times$256 pixels, and then a random crop of 224$\times$224 is selected for training. We train our $n$-BQ-NN for 50 epochs with batch size of 128/16 based on AlexNet/Vgg-16. We use ADAM optimizer with the learning rate of 1e-4. We replace the Local Contrast Renomalization layer with Batch Normalization layer. At inference, we use the 224$\times$224 center crop for forward-propagation.

The ablation experiments are listed in \textbf{Table}~\ref{alexNet}. The baseline AlexNet model scores 56.6\% top-1 accuracy and 80.2\% top-5 accuracy that is reported in \cite{rastegari2016xnor}. For ablation studies, we strictly control the consistency of variables, including network structure, bit width and quantized layers. The only difference is the quantization method. In experiments of ``1-1" v.s. ``1-1", ``1-16" v.s. ``1-32" and ``2-16" v.s. ``2-32", our $n$-BQ-NN achieves 6.9\%, 11.8\% and 0.7\% accuracy improvements respectively. For ``3-16" v.s. ``32-32", our $n$-BQ-NN only reduces the accuracy by 0.6\%.

\section{Acceleration on FPGA}

\begin{figure}[htbp]
	\centering
	\includegraphics[width=0.5\textwidth]{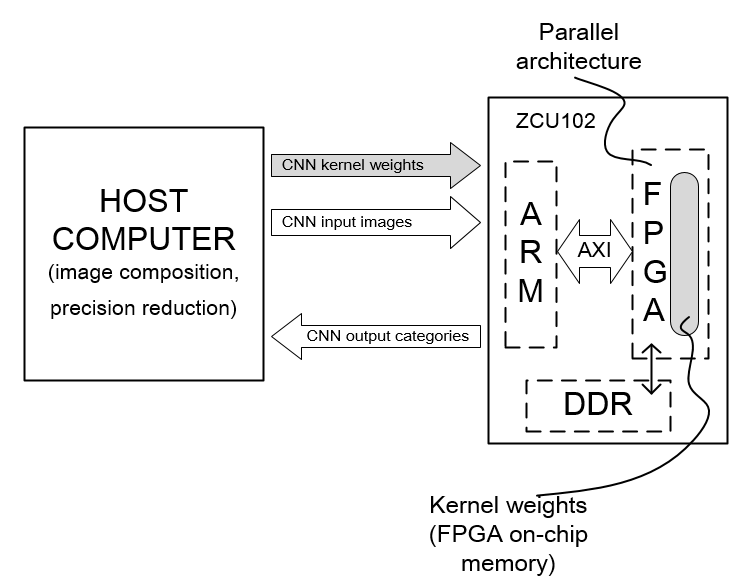}
	\caption{High-level block diagram of our system.}
	\label{high-level} 
\end{figure}

\begin{figure*}[htbp]
	\centering
	\includegraphics[width=0.9\textwidth]{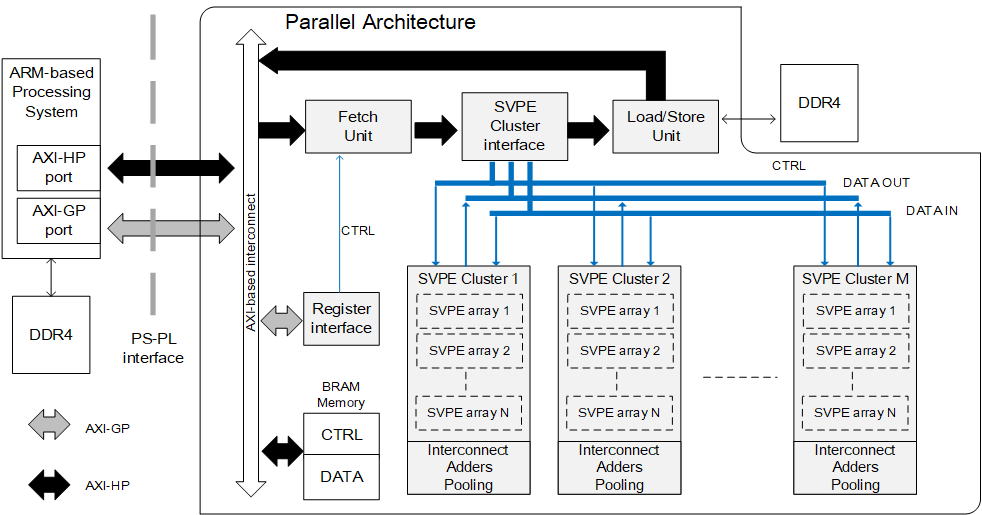}
	\caption{The overall parallel architecture with cluster of shift vector processing elements (SVPEs).}
	\label{architecture} 
\end{figure*}

\begin{figure*}[h]
	\centering
	\includegraphics[width=0.9\textwidth]{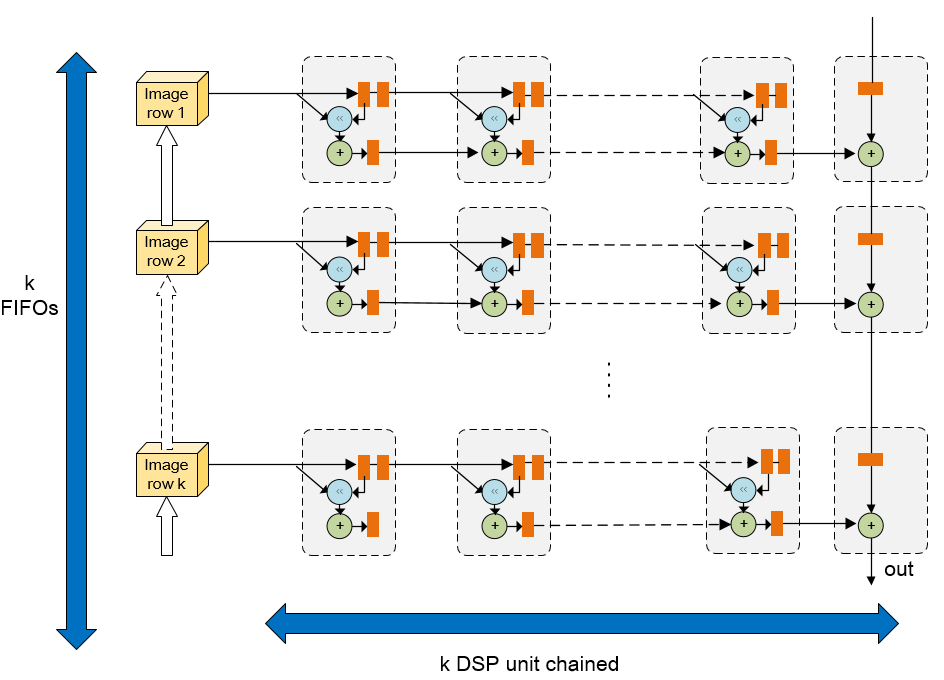}
	\caption{A SVPE array implementing the primitive 2D convolver unit. Where the double orange rectangles represent on-chip memory banks to buffer the weights, the single orange rectangles represent registers, the skyblue roundnesses represent shift operations instead of multipliers in VPE array, and the deepgreen roundnesses represent adders.}
	\label{SVPE_arry} 
\end{figure*}

We evaluate our $n$-BQ-NN on FPGA platform: Xilinx ZCU102, which consists of an UltraScale FPGA, quad ARM Cortex-A53 processors and 500 MB DDR3. To measure the performance of our $n$-BQ-NN running on FPGA, we get the data of run-time, resource utilization, and power by simulating and testing on Vivado-2017 when the operating frequency is 200 MHz. Our $n$-BQ-NN implementation involves a few design parameters, parallelization degree ($P_n$ and $P_m$), filter size ($k$), input feature maps width ($W$), input feature maps height ($H$), input feature channels ($M$), and output feature channels ($N$).

\subsection{Coprocessor Architecture}

\textbf{Fig}.\ref{high-level} shows the block diagram of our system. In the CNN calculation process, the host computer hands off the weights and images to the coprocessor (ZCU102), and collects the predicted classification results. The transmission mode between host computer and ARM CPU can be switched in PCI or UDP. Before the computation, the host computer is responsible for feeding the images and reducing precision. In addition, the ARM CPU needs to complete the calculation of fully-connected layers that is not suitable for FPGA parallel acceleration and the FPGA accelerates the calculation of convolutional layers.

We build the coprocessor with parallel architecture, as shown in \textbf{Fig}.\ref{architecture}. The critical part of the coprocessor is SVPE cluster interface that has $M$ SVPE clusters, where each SVPE cluster consists of $N$ SVPE arrays with a size of $k \times k$. The adders are used to compute partial sums of convolutions while the SVPE arrays compute convolutions. The fetch unit is programmed to fetch images and weights from ARM-based processing system (PS), and the load/store unit is used to load or store intermediate calculation results. The AXI-HP port is used to receive or send the data, and the AXI-GP port is used to receive or send the network structure information and the control signal. A key point to note is 16-bit computational accuracy acts on the data buses to save data bandwidth.

The basic design ideas continue the architecture of $n$-BQ-NN, which converts all 16-bit weights as $n$-bit weights to reduce memory usage and increase the parallelization degree. On FPGAs, due to the shortage of DSPs, this has become a major factor limiting the increase in parallelization degree that directly affects the ability to accelerate calculation for CNNs, because the multiplication in the convolution calculation needs to call DSPs. Instead, we implement the multiplication with SHIFT operation that consumes the Look-Up-Table (LUT) arrays on FPGAs, while the resource of LUTs is more abundant than DSPs' \cite{lu2017evaluating}. In general, our $n$-BQ-NN, which consists of 16-bit activations and $n$-bit weights. 

Our architecture of convolution computation is characterized by several key attributes compared with VPEs \cite{sankaradas2009massively}. First, we organize the architecture as arrays of SVPEs, where the SVPE array is an array of 2D convolvers, each of which consists of $k^2$ connected SHIFT and ADDER units instead of Multiply Accumulate (MAC) units, shown in \textbf{Fig}.\ref{SVPE_arry}. The weights and feature maps are loaded into each PE alternately by AXI-HP port. Before each calculation, the weights are buffered to the specified areas (the double orange rectangles in \textbf{Fig}.\ref{SVPE_arry}), and then the pipeline calculation starts with the enablement of feature maps. Modeling the SVPE and VPE arrays, we compare their resource consumption, parallelization degree and power on FPGAs shown as \textbf{Table}~\ref{comparsion}, and our SVPE array achieves the average energy consumption of 3.81 W at different $n$ that is less than VPE array of 5.53 W. Second, we reduce banded off-chip data memory and improve the data movement between the SVPE clusters and the off-chip memory by using $n$-bit weights. Third, all convolvers are homogeneous when $k$ is fixed as our primitive operator. We evaluate the improvement of the computational efficiency of $n$-BQ-NN in hardware by the following section.

\subsection{Computational Efficiency}

\renewcommand\arraystretch{1.5}
\begin{table}[h]
	\caption{Comparison of VPE and SVPE array resource consumption, parallelization degree and power.}
	\label{comparsion} 
	\begin{center}
		\setlength{\tabcolsep}{1.5mm}{
			\begin{tabular}{|c|c|c|c|c|c|c|c|}
				\hline
				\multicolumn{2}{|c|}{\textbf{Array}}&
			    \multicolumn{5}{c|}{\textbf{SVPE}}&\textbf{VPE}\cr
				\hline
				\multicolumn{2}{|c|}{\textbf{Precision (n)}}&1&2&3&4&5&16 \\
				\hline
				\multirow{4}{*}{\textbf{Power (W)}}&Signal&1.94&2.31&2.61&2.53&2.13&4.88\\ \cline{2-8}
				&Logic&1.03&1.33&1.55&1.63&1.62&0.25\\ \cline{2-8}
				&DSPs&0.08&0.09&0.09&0.08&0.06&0.40\\ \cline{2-8}
				&Total&3.05&3.73&4.25&4.24&3.81&5.53\\ \cline{2-8}
				\hline
				\multirow{3}{*}{\textbf{\makecell*[c]{Used \\ Resource}}}&LUTs&353&280&307&334&346&41 \\ \cline{2-8}
				&FFs&220&226&232&238&244&213 \\ \cline{2-8}
				&DSPs&3&3&3&3&3&12 \\ \cline{2-8}
				\hline
				\multicolumn{2}{|c|}{\textbf{\makecell*[c]{Parallelization degree \\ ($P_n, P_m$)}}}&\multicolumn{5}{c|}{(8,32)}&(4,16)\\
				\hline
		\end{tabular}}
	\end{center}
\end{table}
 
Since the filter size (3$\times$3) is fixed for our $n$-BQ-NN, resource utilization will be maximized. Here, we can predict the performance of $n$-BQ-NN on FPGAs by developing an analytical model. In the following, we rely on it to compare computational efficiency between traditional implementation and $n$-BQ-NN on FPGAs.

On the hardware, MAC unit, adder and multiplier will consume DSP. In fact, the number of DSPs only depends on the size of filter and parallelization degree \cite{lu2017evaluating,peemen2013memory} as follows,
\begin{equation} 
DSP = (k^2 + k) \times P_n \times P_m
\label{eq15}
\end{equation}

We must balance the memory bandwidth between the on-chip and off-chip memory and ensure that the speed of transmission is greater than or equal to the speed of computation for utilizing the resource efficiently. The formula of the time to process input data in the line buffer on FPGA is, 
\begin{equation} 
T_{compute} = (H \times W \times \frac{M}{P_m} \times \frac{N}{P_n}) \times \frac{1}{Freq}
\label{eq16}
\end{equation}
Where $Freq$ is the operating frequency of the FPGA.
Together, we have to parallel the speed of transmission between input and output data as follows,
\begin{equation} 
T_{transfer} = \frac{M \times N \times k^2 + k \times W \times M}{Bandwidth}
\label{eq17}
\end{equation}

We require that $T_{transfer} \leq T_{compute}$. Therefore, we can get the minimum requirement of bandwidth is,
\begin{equation} 
Bandwidth_{min} = \frac{P_m \times P_n}{min(N,M)} \times b_{compute} \times Freq
\label{eq18}
\end{equation}
Where $b_{compute}$ is the bit-width of computation, and we evaluate the performance of hardware acceleration choosing 16 bit-width.
We define the $T_{init}$ as the time to load the first $n$ rows of input image and filter needed into on-chip memory as follows,
\begin{equation} 
T_{init} = \frac{M \times N \times k^2 \times b_{weight}}{Bandwidth}  + \frac{W \times M \times k}{Bandwidth} \times b_{compute}
\label{eq19}
\end{equation}
Where $b_{weight}$ is the bit-width of the weights. The total operations are,
\begin{equation} 
OPs = H \times W \times M \times N \times k^2 \times 2
\label{eq20}
\end{equation}

And the total processing time of the convolution is,
\begin{equation} 
T_{total} = T_{compute} + T_{init}
\label{eq21}
\end{equation}

Finally, we can compare the computational efficiency of the different models defining the effective performance of convolution as,
\begin{equation} 
Perf_{eff} = \frac{OPs}{T_{total}}
\label{eq22}
\end{equation}
We obtain the computational efficiency $Perf_{eff}(n)$ corresponding to different bit width of the weights where $n = b_{weight}$ represents the bit width of our $n$-BQ-NN.
\begin{equation} 
Perf_{eff}(n) = \frac{{32 Freq P_m P_n H W N k^2}}{16 H W N + n M N k^2 + 16 W M k}
\label{final}
\end{equation}

Now, given a convolutional layer represented by ($W$,$H$,$M$,$N$), we get the computational efficiency based on design parameters ($k$,$P_m$,$P_n$).

The main reason for restricting the computational efficiency of CNNs on FPGA is parallelization degree, which is directly related to DSPs when the setting of bandwidth is reasonable. To speed-up the inference of CNNs on FPGAs, we use our SVPE cluster to replace the traditional VPE cluster by converting the multiplications as the SHIFT and ADDER operations. Since we no longer use the multiplication, the amount of DSPs is reduced as follows,
\begin{equation} 
DSP = k \times P_n \times P_m
\label{eq23}
\end{equation}

Since SVPE array consumes much less DSPs than VPE array compared \textbf{Eq}.\ref{eq15} with \textbf{Eq}.\ref{eq23}, $n$-BQ-NN with SVPE array can get a larger amount of parallelization degree than CNNs with VPE array when the consumed DSPs are the same. Based on the maximum DSPs number of 2520 as \textbf{Table}~\ref{n-bit} and the balanced memory bandwidth, we can design the maximum parallelization degree of ($P_m=32$,$P_n=8$) and ($P_m=16$,$P_n=4$) respectively on SVPE and VPE array with filters 3$\times$3. Thanks to the SVPE array, the parallelization degree increases by 4 times to improve the computational efficiency greatly when the total consumed DSPs is 768. Supposing the color image is 32($W$)$\times$32($H$)$\times$3($M$) pixels, filter size $k$ is 3, DDR bit width $N$ is 128, the computational efficiency of our $n$-BQ-NN using SVPE array has improved by about 4.1 times compared with traditional network using VPE array on the basis of \textbf{Eq}.\ref{final}.

\subsection{Performance on FPGA}

\begin{figure}[htbp]
	\centering
	\includegraphics[width=0.5\textwidth]{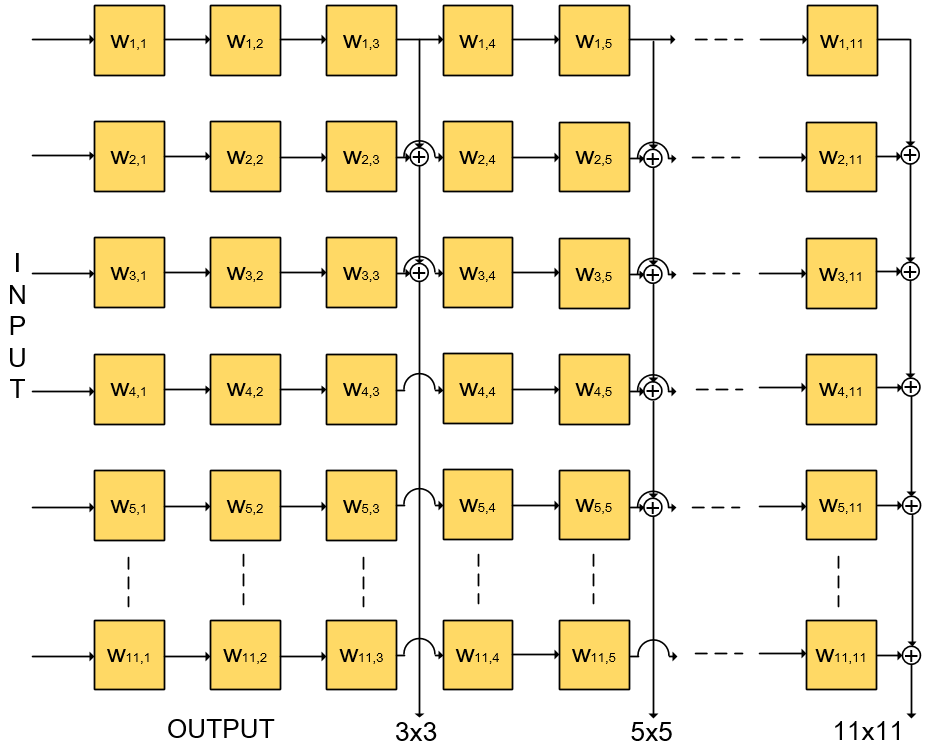}
	\caption{A universal SVPE array supporting AlexNet with multiple kernel sizes (3, 5, and 11), where the yellow blocks are Shift-Accumulator units.}
	\label{universal_SVPE} 
\end{figure}

We evaluate our SVPE cluster implementation using AlexNet, where our $n$-BQ-NN contains 16-bit activations and 3-bit weights. \textbf{Table}~\ref{n-bit} gives the evaluation results with the comparison of the state-of-the-art FPGA accelerators, where GOP indicates the unit of the number of operations. From the hardware perspective, we prefer to use the computational efficiency to describe the performance of the algorithm. Because the total operations of computing a network is fixed, we can get the execution time (s) by dividing the total operations (GOP) by the computational efficiency (GOP/s). Similar to the structure of \textbf{Fig}.\ref{SVPE_arry}, a universal SVPE array designed by the largest filter size of AlexNet is proposed in \textbf{Fig}.\ref{universal_SVPE}. This experiment will use the universal SVPE array that fits the full-size after a slight adjustment. Our array designs to be recycled when calculating the convolutions of different layers, and $11 \times 11$ filter of AlexNet is only used in the first convolutional layer, so most of the array utilization is extremely low. This also confirms the necessity of designing the network architecture with $3 \times 3$ unified filter to improve resource utilization as \textbf{Table}~\ref{TBQ-Net_structure}.

Compared to prior works~\cite{zhang2015optimizing,lu2017evaluating}, we improve the average CNN performance to 957.4 GOP/s where the work~\cite{lu2017evaluating} is implemented by Winograd algorithm. The baseline is to implement the same hardware architecture as our implementation. The only difference is that it uses VPE cluster because its weights and activations are both 16-bit. The computational efficiency of our implementation has improved by 2.9 times compared with the baseline, which is slightly less than 4.1 times based on the theoretical calculations of the Section V-B. On the other hand, Our implementation also improves the energy efficiency to 48.9 GOP/s/W. The better energy efficiency and resource efficiency come from the novel SVPE structure.

\renewcommand\arraystretch{1.5}
\begin{table}[h]
	\caption{Performance comparison for AlexNet}
	\label{n-bit} 
	\begin{center}
		\setlength{\tabcolsep}{1.5mm}{
		\begin{tabular}{|c|c|c|c|c|}
			\hline
			 & \cite{zhang2015optimizing} & \textbf{Baseline} & \cite{lu2017evaluating} & \textbf{Our Impl.} \\
			\hline
			\textbf{Precision} & 32bits fixed & 16bits fixed & 16bits fixed & 16bits fixed \\
			\hline
			\textbf{Device} & VX485T & ZCU102 & ZCU102 & ZCU102 \\
			\hline
			\textbf{Freq(MHz)} & 100 & 200 & 200 & 200 \\
			\hline
			\textbf{Logic cell(K)} & 485.7 & 600 & 600 & 600 \\
			\hline
			\textbf{DSP} & 2800 & 2520 & 2520 & 2520 \\
			\hline
			\textbf{BRAM(Kb)} & 2060$\times$18 & 1824$\times$18 & 1824$\times$18 & 1824$\times$18 \\
			\hline
			\hline
			\textbf{conv1(GOP/s)} & 27.5 & 227.5 & 409.6 & 410.5 \\
			\textbf{conv2(GOP/s)} & 83.8 & 535.8 & 1355.6 & 1744.3 \\
			\textbf{conv3(GOP/s)} & 78.8 & 655.9 & 1535.7 & 1680.7 \\
			\textbf{conv4(GOP/s)} & 77.9 & 634.4 & 1361.7 & 1739.4 \\
			\textbf{conv5(GOP/s)} & 77.6 & 559.5 & 1285.7 & 1456.1 \\
			\hline
			\textbf{\makecell*[c]{CNN average \\ (GOP/s)}} & 61.6 & 332.2 & 854.6 & 957.4 \\
			\hline
			\hline
			\textbf{Power(W)} & 18.6 & 28.7 & 23.6 & 19.6 \\
			\hline
			\textbf{\makecell*[c]{DSP Efficiency \\ (GOP/s/DSPs)}} & 0.022 & 0.131 & 0.339 & 0.381 \\
			\hline
			\textbf{\makecell*[c]{Logic cell \\ Efficiency \\ (GOP/s/cells/K)}} & 0.127 & 0.553 & 1.424 & 1.596 \\
			\hline
			\textbf{\makecell*[c]{Energy \\ Efficiency \\ (GOP/s/W)}} & 3.31 & 11.57 & 36.2 & 48.85 \\
			\hline
			\hline
			\textbf{DSP Utilization} & 80\% & 30\% & 63\% & 30\% \\
			\hline
			\textbf{LUT Utilization} & 61\% & 48\% & 39\% & 73\% \\
			\hline
			\textbf{FF Utilization} & 34\% & 42\% & 33\% & 68\% \\
			\hline
			\textbf{\makecell*[c]{BRAM \\ Utilization}} & 50\% & 50\% & 43\% & 83\% \\
			\hline
		\end{tabular}}
	\end{center}
\end{table}

\section{Conclusion and future work}

In this paper, we present a novel learning framework to quantize full-precision CNN models into low-precision QNN models whose weights are constrained to the power of two. We solve the problem of gradient vanishing by adding a reconstructed gradient function into back-propagation algorithm. To satisfy the network-structure-optimization requirements for hardware limitation, we propose $n$-BQ-NN, a novel QNN structure, to replace the multiplication with SHIFT operation whose structure is more suitable for the inference on FPGAs. Furthermore, we also design the SVPE array to replace all 16-bit multiplications with SHIFT operations in convolution operation on FPGAs. For proving the validity of our learning framework, we conduct experiments and show that the quantized models of ResNet, DenseNet and AlexNet through our learning framework can achieve almost the same accuracies with the original full-precision models. Moreover, when using our learning framework to train our $n$-BQ-NN from scratch, it can achieve nearly state-of-the-art results compared with typically low-precision QNNs. We also evaluate the computational efficiency and energy consumption by implementing our QNNs models on Xilinx ZCU102 platform. In our hardware experiments, our $n$-BQ-NN with our SVPE can execute 2.9 times faster than with the VPE in inference, and the use of SVPE array also reduces average energy consumption to 68.7\% of the VPE array with 16-bit. Our future work should explore how to decrease the accumulated quantization errors further when our learning framework is used on different CNNs structures.


\bibliographystyle{Bibliography/IEEEtranTIE}
\bibliography{Bibliography/IEEEabrv,Bibliography/BIB_1x-TIE-2xxx}\ 

\end{document}